\documentclass[runningheads]{llncs}

\usepackage{eccvabbrv}

\usepackage{graphicx}
\usepackage{booktabs}
\usepackage{booktabs}
\usepackage{multirow}
\usepackage{xcolor}
\usepackage{colortbl}
\usepackage{xspace}
\usepackage{booktabs}
\usepackage{multirow}
\usepackage{adjustbox}
\usepackage{xcolor}
\usepackage{colortbl}

\usepackage[accsupp]{axessibility}  

\usepackage{hyperref}

\usepackage{orcidlink}

\newcommand{\ourmethod}{DiReCT }

\begin{document}

\title{\ourmethod: Disentangled Regularization of Contrastive Trajectories for Physics-Refined Video Generation}

\titlerunning{\ourmethod: Physics-Refined Video Generation}

\author{Abolfazl Meyarian\inst{1}\textsuperscript{\dag} \and
Amin Karimi Monsefi\inst{2}\textsuperscript{\dag} \and \\
Rajiv Ramnath\inst{2} \and 
Ser-Nam Lim\inst{3}}

\authorrunning{A.~Meyarian et al.}


\institute{Path Robotics, Columbus, OH, USA \and
The Ohio State University, Columbus, OH, USA \and
University of Central Florida, Orlando, FL, USA}

\maketitle
\let\oldthefootnote\thefootnote
\let\thefootnote\relax
\footnotetext{\textsuperscript{\dag} Equal contribution.}
\let\thefootnote\oldthefootnote
\setcounter{footnote}{0}

\begin{abstract}

Flow-matching video generators produce temporally coherent, high-fidelity outputs yet routinely violate elementary physics because their reconstruction objectives penalize per-frame deviations without distinguishing physically consistent dynamics from impossible ones. Contrastive flow matching offers a principled remedy by pushing apart velocity-field trajectories of differing conditions, but we identify a fundamental obstacle in the text-conditioned video setting: \emph{semantic-physics entanglement}. Because natural-language prompts couple scene content with physical behavior, naive negative sampling draws conditions whose velocity fields largely overlap with the positive sample's, causing the contrastive gradient to directly oppose the flow-matching objective. We formalize this gradient conflict, deriving a precise alignment condition that reveals when contrastive learning helps versus harms training. Guided by this analysis, we introduce \textbf{\ourmethod} (\textbf{Di}sentangled \textbf{Re}gularization of \textbf{C}ontrastive \textbf{T}rajectories), a lightweight post-training framework that decomposes the contrastive signal into two complementary scales: a \textit{mac-} ro-\textit{contrastive} term that draws partition-exclusive negatives from semantically distant regions for interference-free global trajectory separation, and a \emph{micro-contrastive} term that constructs hard negatives sharing full scene semantics with the positive sample but differing along a single, LLM-perturbed axis of physical behavior; spanning kinematics, forces, materials, interactions, and magnitudes. A velocity-space distributional regularizer helps to prevent catastrophic forgetting of pretrained visual quality. When applied to Wan 2.1-1.3B, our method improves the physical commonsense score on VideoPhy by 16.7\% and 11.3\% compared to the baseline and SFT, respectively, without increasing training time. Additionally, our method achieves the highest total score on WorldModelBench (5.68) among all compared models, surpassing CogVideoX-5B (5.33) while having just $3.8\times$ our parameter count.

\end{abstract}

\section{Introduction}
\label{sec:introduction}

Flow-matching architectures~\cite{zheng2024open,yang2024cogvideox,wan2025wan} have emerged as the dominant framework for video generation, most commonly conditioned on natural-language prompts, producing temporally coherent, high-resolution videos that are increasingly difficult to distinguish from real footage on semantic quality alone.
Yet a conspicuous gap remains between \emph{looking real} and \emph{behaving real}: bouncing balls accelerate after impact, and colliding objects inter-penetrate as if neither were solid.
These failures are systematic, with benchmarks reporting that a majority of generated videos depicting non-trivial physical interactions contain at least one clearly implausible event~\cite{physbench,t2vbench}, and consequential: when such generators are used as predictive world models~\cite{unisim,worldmodel_survey}, even a single physics violation can cascade into incorrect downstream decisions.

\begin{figure}[t]
    \centering
    \includegraphics[width=\textwidth]{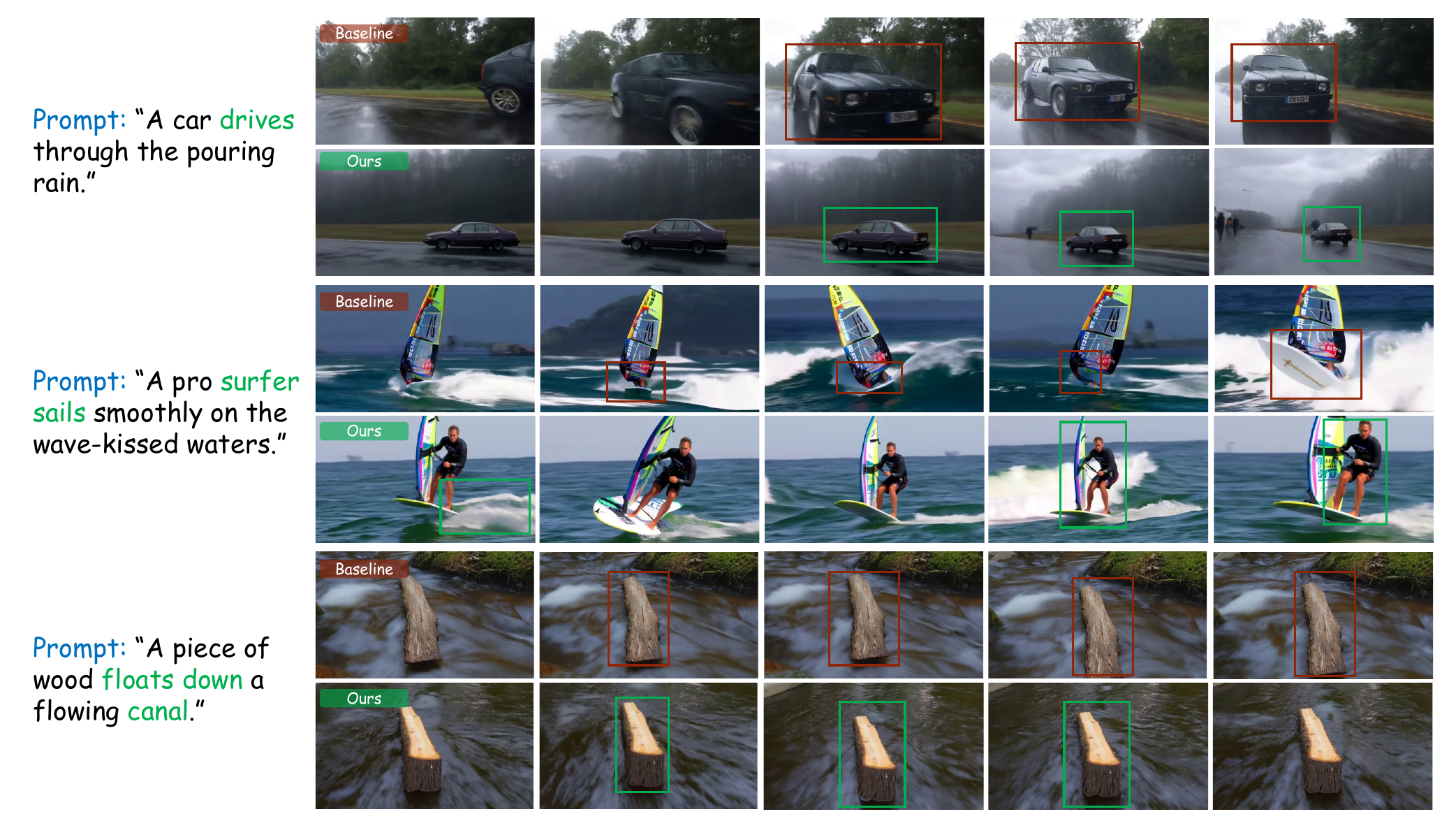}
    \caption{Comparison of zeroshot Wan-2.1-1.3B against the same model trained using \ourmethod on a few prompts from VideoPhy. In the top example, the baseline drives the car backward, violating forward kinematics, while \ourmethod produces a consistent forward trajectory. In the middle example, the baseline's surfer progressively merges with the sail rig, losing bodily structure, whereas \ourmethod maintains the surfer's integrity and mass throughout. In the bottom example, the baseline keeps the wood stationary and upright despite the flowing current, while \ourmethod generates realistic downstream motion consistent with buoyancy and flow dynamics.}
    \label{fig:oursvsbaseline}
\end{figure}

A major contributor is the training objective itself. Reconstruction-based losses, pixel-level $\ell_2$, latent distances, and conditional flow-matching objectives ~\cite{lipman2022flow}, penalize per-frame deviations without distinguishing physically consistent dynamics from impossible ones.
Prior work addresses this gap either by augmenting training with synthetic physics rollouts~\cite{liu2024physgen,lv2024gpt4motion,yang2025vlipp}—for example, \cite{lv2024gpt4motion} is confined to the three material types (rigid bodies, cloth, and liquids) that Blender's built-in physics engine can simulate—requiring domain-specific engineering as well as introducing distribution shift, or by injecting explicit physical priors into the loss or architecture~\cite{shao2025finephys,wang2025wisa}, which improves targeted phenomena but struggles to generalize across diverse dynamics.

An analysis of the velocity-field landscape reveals a deeper structural cause. As conditioning prompts become more similar, either along the \emph{semantic} axis (what is depicted) or the \emph{physical} axis (how it moves), the corresponding flow trajectories converge, and the learned velocity field collapses toward the mean of nearby trajectories~\cite{deltafm,guo2025variational,zhang2025towards}.
This produces outputs that default to physically generic behaviors.
A natural remedy is to add a contrastive signal to the flow-matching objective, pushing the current trajectory away from those whose conditions differ in critical ways~\cite{deltafm}.

However, contrastive flow matching was developed for settings where conditions are \emph{discrete} and \emph{well-separated},  e.g., class labels in image generation, where it yields substantial quality gains~\cite{deltafm}.
In text-conditioned video generation, conditions are \emph{continuous} and \emph{entangled}: a single caption couples semantic content (what the scene contains) with physical dynamics (how objects move), making the two axes covary and impossible to isolate from the caption alone.
A naive contrastive objective, one that simply pushes the velocity field of a training prompt (the \emph{positive}) away from that of a randomly drawn alternative (the \emph{negative}), therefore conflates two goals: separating flows for different visual concepts (already handled adequately by reconstruction) and distinguishing flows governed by different physics (the actual objective).
We show in Section~\ref{subsec:hardvsrand} that when positive and negative conditions share substantial velocity-field structure, as inevitably happens under semantic proximity, the contrastive gradient \emph{opposes} the flow-matching gradient, degrading the overall performance.

We introduce \textbf{\ourmethod} (\textbf{Di}sentangled \textbf{Re}gularization of \textbf{C}ontrastive \textbf{T}- rajectories), a lightweight post-training framework. \ourmethod resolves this conflict via \emph{entanglement-aware multi-scale contrastive learning} in velocity space. Our method decomposes the contrastive objective into two complementary scales. 

A \emph{macro-contrastive} term draws semantically distant negatives via partition-exclusive sampling, providing a clean global separation signal free from gradient interference.
A \emph{micro-contrastive} term targets fine-grained physics distinctions by constructing hard negatives that share scene semantics with the positive but differ along a single, controlled axis of physical behavior, generated via minimal, LLM-guided prompt perturbations across five physics. Additionally, a velocity-space KL penalty prevents catastrophic forgetting of the pretrained model's visual quality.
%
%
In summary, our contributions are:
\begin{itemize}
  \item \textbf{Formalizing gradient conflict in text-conditioned contrastive flows.}
  
  We identify and formalize a failure mode of contrastive flow matching under rich text conditioning: semantic--physics entanglement causes the contrastive gradient to oppose the flow-matching gradient when positive and negative share velocity-field structure, establishing a precise condition under which naive negative sampling degrades training.

  \item \textbf{\ourmethod: entanglement-aware multi-scale contrastive post-training.}
  We introduce \ourmethod, a lightweight post-training framework that incurs no additional training cost over standard SFT, which decomposes the contrastive objective into a macro term (partition-exclusive random negatives for global trajectory separation) and a micro term (physics-perturbed hard negatives for fine-grained physical distinction), with each scale addressing a complementary regime of the gradient-conflict spectrum.

  \item \textbf{Comprehensive evaluation across physics regimes.}
  We validate \ourmethod on physics-oriented benchmarks spanning diverse interaction categories (e.g., collisions, gravity, deformation, fluids), demonstrating improvements in physical plausibility without sacrificing visual quality, and provide ablations isolating the impact of entanglement-aware sampling, hard-negative generation, and each contrastive scale.
\end{itemize}

\section{Related Work}
\label{sec:related_work}

Flow matching~\cite{lipman2022flow,liu2023rectifiedflow}, combined with Transformer backbones~\cite{peebles2023dit,esser2024sd3,ma2024sit}, has demonstrated strong results across diverse generative tasks~\cite{ICLR2025_a44a70ac,nie2025large,monsefi2026fsdfmfastaccuratelong} and currently underpins most open video generators~\cite{yang2024cogvideox,wan2025wan,zheng2024open}.
Recent advances have further improved efficiency via pyramidal denoising~\cite{jin2024pyramidal} and representation alignment~\cite{yu2024representation}.
Despite rapid gains in visual quality, the velocity-regression objective treats all token-level errors equally; it does not distinguish trajectories that comply with physical laws from those that violate them.

A growing line of work targets this gap. Simulation-guided methods~\cite{liu2024physgen,lv2024gpt4motion} inject explicit physics priors but are restricted to the phenomena their simulator covers and can introduce distribution shift. Independent evaluations~\cite{kang2024far} further suggest that scaling data and parameters alone yield case-based rather than universal physical reasoning, motivating physics-specific post-training. Inference-time approaches such as PhyT2V~\cite{xue2025phyt2v} refine prompts via LLM chain-of-thought, improving plausibility at the cost of substantial generation overhead. On the training side, WISA~\cite{wang2025wisa} fine-tunes Mixture-of-Physical-Experts Attention on a curated 80K-video corpus with fine-grained physics prompts; PISA~\cite{li2025pisa} demonstrates the value of even minimal simulation for supervision but reveals poor out-of-distribution transfer beyond its free-fall training domain; and PhysMaster~\cite{ji2025physmaster} achieves strong per-phenomenon results via multi-stage DPO at the cost of dedicated human annotations for each physical category. A recurring pattern connects these efforts: each couples improvement to a domain-specific resource—a simulator, a categorized reward model, or per-phenomenon labels—whose coverage bounds the resulting generalization. This motivates a contrastive objective that derives its physics signal directly from the velocity field, requiring only coarse category-level labels rather than per-phenomenon infrastructure (Section~\ref{sec:multi_scale}).

Contrastive objectives have improved several generative pipelines~\cite{park2020contrastive,kang2020contragan,zhu2022discrete,meral2024conform,dalva2024noiseclr}. Most directly relevant, $\Delta$FM~\cite{deltafm} augments the flow-matching loss with a contrastive regularizer that pushes apart predicted velocities for different conditions. A revealing asymmetry appears in their results: the class-conditional ImageNet variant, where labels are discrete and well-separated, yields significant visual improvements, whereas the text-conditioned CC3M variant produces notably lower quality gains. While $\Delta$FM demonstrates the benefit of contrastive regularization, they do not analyze the failure mode that arises when conditions share substantial velocity-field structure. We hypothesize that this asymmetry stems from semantic entanglement: rich captions couple visual content with physical and stylistic attributes, so a contrastive loss applied without regard to this coupling conflates all axes of variation. We formalize this hypothesis for the video setting in Appendix~\ref{append:gradient_conflict}, deriving a precise condition under which the contrastive gradient opposes the flow-matching objective, and verify it empirically via gradient alignment measurements in Appendix~\ref{append:gradient_alignment}. This analysis motivates negative-sampling strategies that respect the entangled structure of the text-conditioning space.

Separately, a rich literature disentangles content from motion at the architectural level: factored latent spaces~\cite{denton2017unsupervised,tulyakov2018mocogan}, decomposed noise~\cite{luo2023videofusion}, separate diffusion processes~\cite{yu2024efficient}, modular temporal modules~\cite{guo2023animatediff}, and dual-path LoRAs~\cite{zhao2024motiondirector}. These approaches effectively decompose what moves from how it moves within the model's internal representations. However, architectural decomposition does not address entanglement within the text conditioning itself: two prompts describing the same scene with different physical outcomes (e.g., a ball bouncing elastically vs. splattering on impact) produce nearly identical text embeddings yet demand distinct velocity fields. This conditioning-level entanglement requires a corresponding intervention in how conditions are contrasted during training — the focus of our approach.

RLHF-based approaches~\cite{xu2023imagereward,li2024t2v,prabhudesai2024video,sorokin2025imagerefl} and Direct Preference Optimization (DPO) variants for images~\cite{wallace2024diffusion,zhu2025dspo} and videos~\cite{liu2025videodpo,liu2025improving,jiang2025huvidpo} have become the dominant paradigm for aligning generative models with human preferences.
However, both paradigms compress the rich spatiotemporal structure of generation errors into coarse preferred-versus-rejected judgments, making it difficult for the model to localize \emph{where} and \emph{why} a trajectory deviates from physical plausibility.
\ourmethod provides a structurally richer training signal: contrastive velocity differences encode spatiotemporal information about \emph{how} the predicted trajectory must change to better respect physical dynamics, operating at the level of the velocity field rather than a scalar preference.

\section{Preliminaries and Backgrounds}
\label{sec:preliminaries}
We develop our proposed framework based on flow matching models~\cite{zheng2024open,yang2024cogvideox,wan2025wan} given their proven superior capabilities in video generation\cite{wan2025wan}. More specifically, to force separation of learned flow trajectories, we follow the $\Delta$FM method introduced in \cite{deltafm}.  

Stoica~\etal~\cite{deltafm} observe that while unconditional flow matching guarantees unique flows between sample pairs, this uniqueness breaks down in conditional settings, flows from different conditions may overlap, producing ambiguous, ``averaged'' generations. To address this, they propose Contrastive Flow Matching ($\Delta$FM), which augments the standard conditional flow matching objective with a contrastive regularization term that explicitly enforces flow uniqueness across conditions.

Conditional flow matching trains a model $v_\theta(x_t, t, y)$ to regress the target velocity $\mathbf{u}^{+} = \dot{\alpha}_t \hat{x} + \dot{\sigma}_t \epsilon$ for a sample $\hat{x} \sim p(x|y)$, where $y$ is the condition, and the noise $\epsilon \sim \mathcal{N}(0, \mathbf{I})$. $\Delta$FM introduces a negative sample $\tilde{x} \sim p(x|\tilde{y})$, where $\tilde{y}$ may or may not equal $y$, paired with an independent noise $\tilde{\epsilon} \neq \epsilon$. The contrastive regularization then maximizes the distance between the predicted velocity $v_{\theta}(x_t, t, y)$ and the independent target velocity $\mathbf{u}^{-} = \dot{\alpha}_t \tilde{x} + \dot{\sigma}_t \tilde{\epsilon}$. The combined objective is:

\begin{equation}
    \mathcal{L}^{(\Delta\text{FM})}(\theta) = \mathbb{E}\Big[\|v_\theta(x_t, t, y) - \mathbf{u}^{+}\|^2 - \lambda \|v_\theta(x_t, t, y) - \mathbf{u}^{-}\|^2\Big],
    \label{eq:delta_fm}
\end{equation}
where $\lambda \in [0,1)$ controls the strength of the contrastive term. The first term is the standard flow matching loss, encouraging the predicted velocity to match the target flow. The objective reduces to standard flow matching when $\lambda = 0$.

\section{Entanglement-Aware Contrastive Flow Matching}\label{sec:multi_scale}

\begin{figure*}[t]
    \centering
    \includegraphics[width=\textwidth]{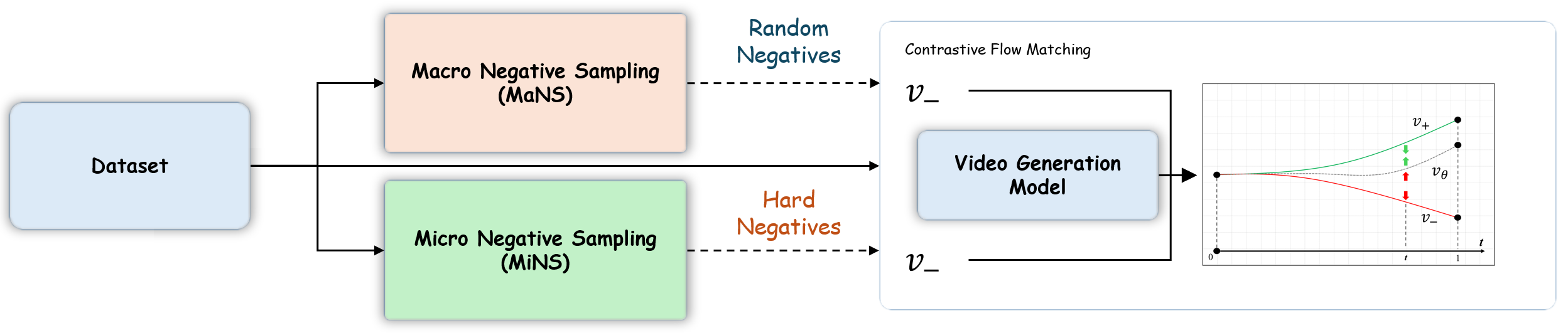}
    \caption{\textbf{Overview of the \ourmethod framework.} The contrastive objective is decomposed into two complementary scales: a \emph{macro-contrastive} term draws random negatives from semantically distant clusters (MaNS), providing clean global trajectory separation, while a \emph{micro-contrastive} term uses physics-perturbed hard negatives (MiNS) that share scene semantics but violate a targeted physical dimension, enabling fine-grained physics discrimination.}
    \label{fig:main_method}
\end{figure*}

Contrastive flow matching~\cite{deltafm} draws its negative samples uniformly at random from within the training batch, without regard to the semantic relationship between the positive condition $y$ and the sampled negative condition $\tilde{y}$.
This design is well-suited to class-conditional generation, where labels are discrete and non-overlapping.
In text-conditioned video generation, however, conditions lie in a continuous embedding space where semantically related prompts---and their associated velocity fields---share substantial structure.
We formalize this failure mode in Appendix~\ref{append:gradient_conflict} and present a two-scale solution in the remainder of this section.

\subsection{Macro-Contrastive Learning (MaNS)}\label{sec:macro}

\begin{figure*}[t]
    \centering
    \includegraphics[width=\textwidth]{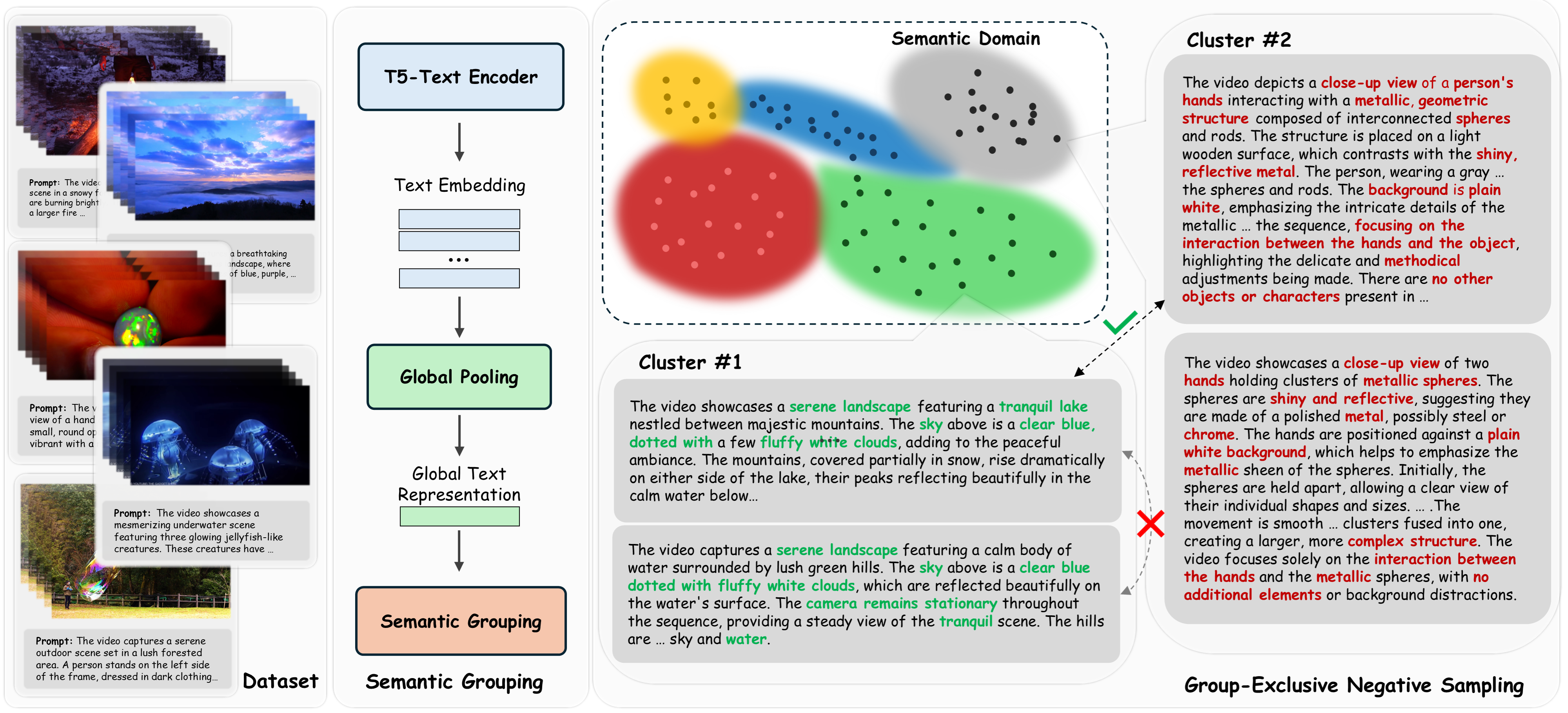}
    \caption{\textbf{Macro-contrastive negative sampling (MaNS).} Prompts are encoded via the text encoder of the video generative model, globally pooled, and partitioned into $K$ semantic regions. For each positive, negatives are drawn exclusively from different partitions, ensuring minimal velocity-field overlap and preventing the gradient conflict identified in Proposition~\ref{prop:conflict} from Appendix~\ref{append:gradient_conflict}.}
    \label{fig:macro}
\end{figure*}


Let $\boldsymbol{\delta} = \mathbf{u}^{+} - \mathbf{u}^{-}$ denote the velocity gap between the positive and negative targets. When $\|\boldsymbol{\delta}\|$ is large---i.e., the positive and negative conditions produce well-separated velocity fields---the contrastive and flow-matching gradients are naturally aligned (Proposition~\ref{prop:conflict}, Appendix~\ref{append:gradient_conflict}). To ensure all random negatives fall in this regime, we partition the condition space into semantically coherent regions and restrict sampling to cross-region pairs. An overview of this mechanism is depicted in Figure~\ref{fig:macro}.

Let $\mathcal{E}: \mathcal{Y} \rightarrow \mathbb{R}^d$ denote the frozen text encoder of the generative model.
Since prompts vary in token length, we obtain a fixed-dimensional representation for each prompt $y$ by globally pooling its token-level features: $\mathbf{z}_y = \mathrm{Pool}(\mathcal{E}(y)) \in \mathbb{R}^d$.
We partition the resulting embedding manifold into $K$ semantic regions on $\{\mathbf{z}_y\}_{y \in \mathcal{D}}$, yielding a partition $\{\mathcal{C}_1, \ldots, \mathcal{C}_K\}$.
Each partition captures a neighborhood of semantically coherent conditions whose velocity fields share substantial structure; the partition boundaries delineate the scale at which semantic overlap begins to cause gradient conflict.

For an anchor condition $y \in \mathcal{C}_i$, the macro-contrastive sampling rule restricts negatives to a different cluster:
\begin{equation}\label{eq:macro_sampling}
    \tilde{y}_{\mathrm{rand}} \sim \mathcal{U}~\!(\bigcup_{j \neq i} \mathcal{C}_j).
\end{equation}
By construction, such negatives share minimal velocity-field structure with the positive ($\|\boldsymbol{\delta}\|$ is large), and the resulting contrastive gradients satisfy the alignment condition (Equation~\ref{eq:alignment_condition} from Appendix~\ref{append:gradient_conflict}).
The macro-contrastive loss is:
\begin{equation}\label{eq:loss_rand}
    \mathcal{L}_{\mathrm{rand}} = \|v_\theta(x_t, t, y) - (\dot{\alpha}_t \tilde{x}_{\mathrm{rand}} + \dot{\sigma}_t \tilde{\epsilon}_{\mathrm{rand}})\|^2,
\end{equation}
where $\tilde{x}_{\mathrm{rand}}$ is a sample drawn from the partition-excluded pool with prompt $\tilde{y}_{\mathrm{rand}}$, and $\tilde{\epsilon}_{\mathrm{rand}}$ is an independently drawn noise vector.
This term teaches the model to broadly reject irrelevant trajectories, thereby establishing global structure in the learned velocity field with minimal interference from semantically entangled conditions.

\subsection{Micro-Contrastive Learning (MiNS)}\label{sec:micro}

\begin{figure*}[t]
    \centering
    \includegraphics[width=\textwidth]{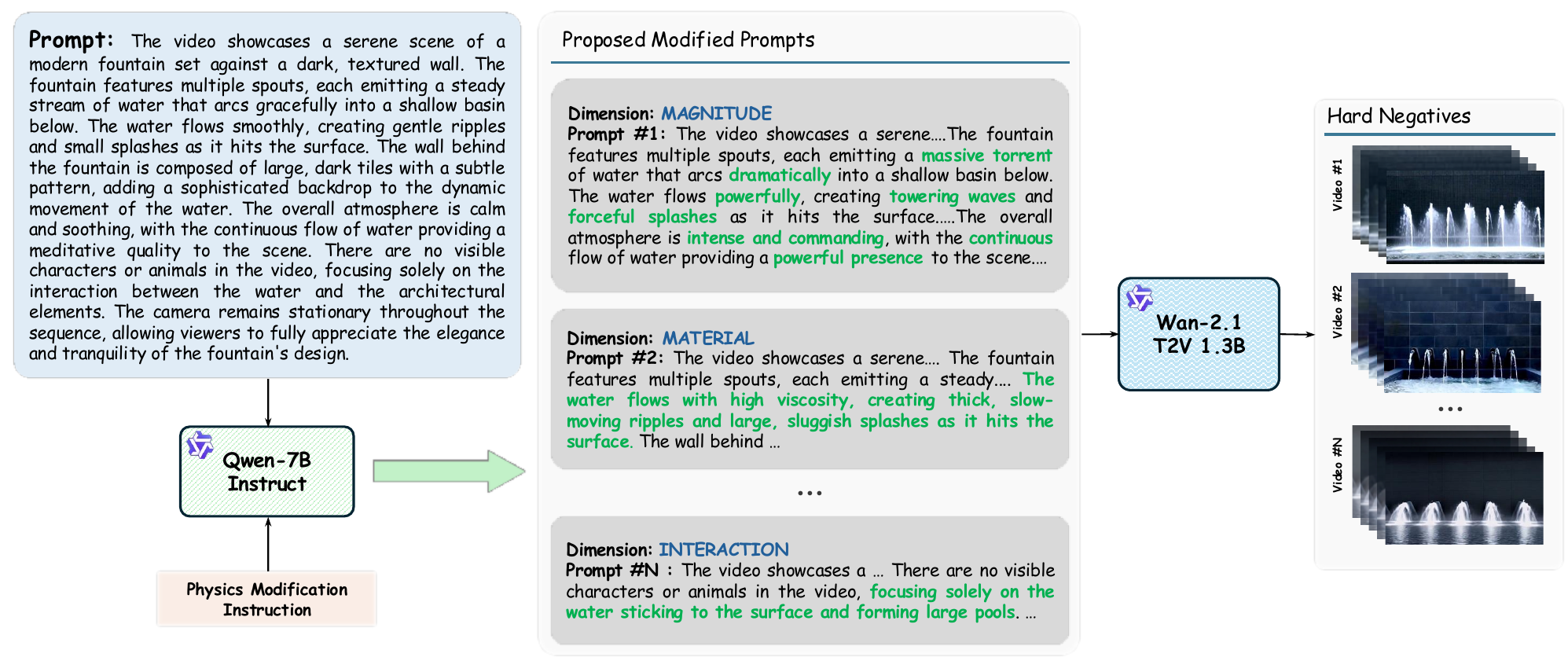}
    \caption{Micro-contrastive hard negative generation (MiNS). For each anchor prompt, an LLM (Qwen2.5-7B-Instruct) perturbs a single physics dimension while preserving scene semantics. The perturbed prompt is rendered by the base model to produce a hard negative video whose velocity trajectory is physically inconsistent with the anchor.}
    \label{fig:micro}
\end{figure*}

The macro-contrastive term establishes global trajectory separation but cannot resolve fine-grained physical distinctions between semantically similar conditions ---precisely the regime where $\|\boldsymbol{\delta}\|$ is small and naive contrastive sampling fails.
The micro-contrastive term addresses this by constructing hard negatives where the velocity gap $\boldsymbol{\delta}$, though small in norm, is \emph{concentrated on the physics-relevant subspace}: anchor and negative share scene semantics but differ along a single, controlled axis of physical behavior.

\paragraph{Physics-axis perturbation.}
We define five perturbation dimensions grounded in classical mechanics:
\begin{itemize}
  \item \textbf{Kinematics}: motion time-profile (e.g., sudden vs.\ gradual),
  \item \textbf{Forces}: dominant force direction or type (e.g., falling vs.\ rising),
  \item \textbf{Material}: a single substance property (e.g., viscosity, elasticity, friction),
  \item \textbf{Interaction}: contact response (e.g., bounce vs.\ stick, shatter vs.\ dent),
  \item \textbf{Magnitude}: scalar intensity of a physical quantity (e.g., droplet vs.\ stream).
\end{itemize}
These five axes are chosen to span two complementary aspects of physical plausibility: how an object \emph{appears} (Material, Magnitude) and how it \emph{interacts with its environment} (Kinematics, Forces, Interaction). Together, they cover the full arc from object properties through dynamic response, ensuring that single-axis perturbations can target either intrinsic physical attributes or emergent behaviors upon contact. 

For each anchor prompt $y$, we randomly sample one dimension and use an LLM (Qwen2.5-7B-Instruct)\footnote{Qwen2.5-7B-Instruct was chosen as it is both efficient for offline generation of perturbed prompts and sufficiently accurate for structured single-axis edits governed by a fixed template.} to produce a corrupted variant $\tilde{y}_{\mathrm{hard}}$, modifying only the minimal text necessary to introduce the targeted physical violation while preserving the scene, objects, setting, and camera style (the full prompt template and per-dimension definitions are provided in Appendix~\ref{append:perturbation_examples}).
This process is repeated $N$ times per prompt, with random-dimension subsampling across the dataset, ensuring that all five dimensions are well represented in aggregate while introducing combinatorial diversity.

MiNS requires that $\tilde{y}_{\mathrm{hard}}$ preserves scene semantics while modifying only the targeted physics axis. We verify this by measuring cosine similarity between anchor and perturbed embeddings in the frozen text-encoder space. Pairs falling below a certain similarity threshold are discarded during dataset construction. Representative positive prompt and perturbation pairs, as well as details of the selection process, are provided in Appendix~\ref{append:perturbation_examples}. We then generate a video for each perturbed prompt using the base model.

\paragraph{Dual benefit of generated negatives.}
Because the hard negatives are synthesized by the base model itself, they exhibit both the targeted physical violation and the characteristic artifacts of the generator's current distribution (e.g., temporal flickering, texture drift).
The micro-contrastive term, therefore, trains the model to push its velocity field away from both physically implausible dynamics and its own generation artifacts simultaneously.
This dual signal is a structural advantage over using real video negatives, which could only provide the physics-violation contrast without addressing generation-specific failure modes.

It is worth noting that the contrastive term acts on the velocity field induced by the conditioning $\tilde{y}_{\mathrm{hard}}$, not on the visual realism of the rendered negative. Even if the base model does not faithfully depict the perturbed physics, the conditioning shift $y \to \tilde{y}_{\mathrm{hard}}$ produces a distinct target velocity in latent space, which is sufficient for the contrastive objective. 

The micro-contrastive loss mirrors the macro term:

\begin{equation}\label{eq:loss_hard}
    \mathcal{L}_{\mathrm{hard}} = \|v_\theta(x_t, t, y) - (\dot{\alpha}_t \tilde{x}_{\mathrm{hard}} + \dot{\sigma}_t \tilde{\epsilon}_{\mathrm{hard}})\|^2,
\end{equation}
where $\tilde{x}_{\mathrm{hard}}$ is the video generated from the physics-perturbed prompt $\tilde{y}_{\mathrm{hard}}$, and $\tilde{\epsilon}_{\mathrm{hard}}$ is an independently drawn noise vector.
Because the perturbation is restricted to a single physics dimension, $\boldsymbol{\delta}$ for this pair is concentrated in the physics-relevant subspace of the velocity field, ensuring the separation term in Equation~\ref{eq:alignment}(from Appendix~\ref{append:gradient_conflict}) remains meaningful despite the high semantic similarity between $y$ and $\tilde{y}_{\mathrm{hard}}$. More details on this matter can be found in Remark~\ref{rem:micro_alignment}.

\subsection{Training Objective}\label{sec:objective}

\paragraph{Distributional anchoring.}
Contrastive post-training risks catastrophic forgetting: the model may achieve strong trajectory separation by drifting away from the pretrained distribution, degrading visual quality, and temporal coherence.
To prevent this, we maintain a frozen copy of the pretrained model $\theta_{\mathrm{ref}}$ and penalize velocity-space divergence from it:
\begin{equation}\label{eq:kl_anchor}
    \mathcal{L}_{\mathrm{anchor}} = \| v_\theta(x_t, t, y) - v_{\theta_{\mathrm{ref}}}(x_t, t, y) \|^2.
\end{equation}
This term acts as a velocity-space analogue of the KL penalty in RLHF~\cite{xu2023imagereward}: it upper-bounds the divergence between the path-measure distributions induced by $\theta$ and $\theta_{\mathrm{ref}}$, ensuring the fine-tuned model remains close to the pretrained model's generation manifold while allowing targeted physics improvements.

\paragraph{Combined objective.}
The full \ourmethod loss integrates reconstruction, both contrastive scales, and distributional anchoring:

\begin{equation}\label{eq:\ourmethod_loss}
\begin{split}
    \mathcal{L}_{\mathrm{\ourmethod}} = \mathbb{E}\Big[
      &\underbrace{\|v_\theta(x_t, t, y) - (\dot{\alpha}_t \hat{x} 
      + \dot{\sigma}_t \epsilon)\|^2}_{\text{flow matching}} \\
      &- \lambda_{\mathrm{rand}}\,\mathcal{L}_{\mathrm{rand}}
      - \lambda_{\mathrm{hard}}\,\mathcal{L}_{\mathrm{hard}}
      + \lambda_{\mathrm{anc}}\,\mathcal{L}_{\mathrm{anchor}}
    \Big],
\end{split}
\end{equation}
where $\lambda_{\mathrm{rand}}$, $\lambda_{\mathrm{hard}} \in [0, 1)$ control the strength of the MaNS and MiNS, respectively, and $\lambda_{\mathrm{anc}} > 0$ controls the distributional anchoring strength.
The objective reduces to standard flow matching when $\lambda_{\mathrm{rand}} = \lambda_{\mathrm{hard}} = \lambda_{\mathrm{anc}} = 0$.
The interplay between the three terms is as follows: the flow-matching term provides the reconstruction signal; the macro-contrastive term sharpens global trajectory separation in cooperative gradient regions; the micro-contrastive term refines physics-specific distinctions in the fine-grained regime; and the anchoring term regularizes the entire process against distributional drift.
A sensitivity analysis of $\lambda_{\mathrm{rand}}$, $\lambda_{\mathrm{hard}}$, and $\lambda_{\mathrm{anc}}$ is provided in Appendix~\ref{append:lambda_sensitivity}.

\section{Experiments} \label{sec:experiments} 
We evaluate \ourmethod from different aspects; additional ablations and analyses are provided in the appendix.
\subsection{Implementation Details}
\label{sec:impl}
\paragraph{Base model and generation backbone.}
We adopt Wan-2.1-T2V-1.3B~\cite{wan2025wan} as both the pretrained backbone for \ourmethod fine-tuning and the generator used to synthesize hard negative videos. All generated videos have a spatial resolution of $480 \times 832$ and a temporal extent of 81~frames. 

\paragraph{Training Configuration.}
We train with the AdamW optimizer using a constant learning rate of $1\times10^{-6}$, $\beta_1 = 0.9$, $\beta_2 = 0.999$. Training proceeds for 15k steps with a global batch size of 8 (4 per GPU) using DeepSpeed ZeRO Stage~2 for distributed training. The frozen reference model $\theta_{\mathrm{ref}}$ is maintained in memory alongside the trainable copy to compute the distributional anchoring term (Equation~\ref{eq:kl_anchor}) at each step. We use $\lambda_{anc}=0.2$, $\lambda_{hard}=0.02$, and $\lambda_{rand}=0.005$. We cluster the training prompts into $K = 32$ semantic groups via $k$-means on frozen text-encoder embeddings. For additional information regarding the hyperparameters, refer to Appendix~\ref{append:lambda_sensitivity} and ~\ref{append:clustering}.

\begin{table*}[t]
    \centering
    \caption{Ablation on negative sampling strategy. All fine-tuned variants use the same base model and training data. VideoPhy measures the physical plausibility of generated videos. WorldModelBench evaluates physical commonsense and physics adherence across diverse scenarios. Best results are shown in \textbf{bold}, second best \underline{underlined}.}
    \label{tab:negative_ablation}
    \renewcommand{\arraystretch}{1.3}
    \resizebox{\textwidth}{!}{%
    \begin{tabular}{lccccccccc}
        \toprule
        \multirow{3}{*}{\textbf{Method}} & \multicolumn{3}{c}{\textbf{VideoPhy(\%)} $\uparrow$} & \multicolumn{6}{c}{\textbf{WorldModelBench} $\uparrow$} \\
        \cmidrule(lr){2-4} \cmidrule(lr){5-10}
         & \multirow{2}{*}{SA} & \multirow{2}{*}{PC} & \multirow{2}{*}{AVG} & \multirow{2}{*}{Instr.} & \multicolumn{2}{c}{Common Sense} & \multicolumn{2}{c}{Physics Adherence} & \multirow{2}{*}{Total} \\
        \cmidrule(lr){6-7} \cmidrule(lr){8-9}
         & & & & & Frame & Temp. & Mass & Penetr. & \\
        \midrule
        \multicolumn{10}{l}{\textit{Baselines}} \\[1pt]
        Zero-shot                  & 50.45 & 32.69 & 41.57 & \textbf{2.18} & 0.93 & 0.80 & 0.72 & 0.83 & 5.46 \\
        + SFT                      & 50.38 & 34.28 & 42.33 & 2.14 & 0.92 & 0.81 & 0.75 & 0.85 & 5.47 \\
        + Random negatives ($\Delta$FM) & 47.71 & 35.89 & 41.80 & 2.12 & 0.91 & 0.80 & 0.73 & 0.84 & 5.40 \\
        \midrule
        \multicolumn{10}{l}{\textit{\ourmethod components}} \\[1pt]
        + MaNS                     & 50.42 & 36.03 & 43.23 & 2.15 & 0.93 & 0.82 & 0.77 & 0.88 & 5.54 \\
        + MiNS                     & 50.10 & 36.17 & 43.14 & 2.14 & 0.93 & 0.83 & 0.74 & 0.85 & 5.50 \\
        + MaNS + MiNS              & \underline{50.78} & \underline{37.46} & \underline{44.12} & 2.15 & \underline{0.93} & \underline{0.85} & \underline{0.78} & \underline{0.89} & \underline{5.60} \\
        \rowcolor[HTML]{ECF4FF} \textbf{+ SFT + MaNS + MiNS (\ourmethod)} & \textbf{51.26} & \textbf{38.16} & \textbf{44.71} & \underline{2.17} & \textbf{0.94} & \textbf{0.88} & \textbf{0.79} & \textbf{0.90} & \textbf{5.68} \\
        \bottomrule
    \end{tabular}%
    }
\end{table*}

\subsection{Datasets}
\label{sec:dataset}
We train \ourmethod on a curated subset of the WISA-80K dataset~\cite{wang2025wisa}, which provides videos annotated across 17 physical laws spanning dynamics, thermodynamics, and optics. Each video is paired with a scene-level caption and a physics-focused description; we use only the scene-level caption as the text condition, ensuring the model infers physical plausibility from visual dynamics rather than explicit physics cues. The coarse physics law category for each sample is also used in our proposed MaNS. From the full dataset, we select 11,299 samples with a minimum of 81 frames and a 16:9 aspect ratio. Longer videos are uniformly subsampled to exactly 81 frames.

\begin{table*}[t]
    \centering
    \caption{Comparison with state-of-the-art open-source video generation models of comparable scale. We report VideoPhy (SA, PC, AVG), WorldModelBench scores across instruction following, common sense, and physics adherence. Best results in \textbf{bold}, second best \underline{underlined}.}
    \label{tab:sota_comparison}
    \renewcommand{\arraystretch}{1.3}
    \resizebox{\textwidth}{!}{%
    \begin{tabular}{lcccccccccc}
        \toprule
        \multirow{3}{*}{\textbf{Model}} & \multirow{3}{*}{\textbf{Params}} 
        & \multicolumn{3}{c}{\textbf{VideoPhy (\%)} $\uparrow$} 
        & \multicolumn{6}{c}{\textbf{WorldModelBench} $\uparrow$} \\
        \cmidrule(lr){3-5} \cmidrule(lr){6-11}
        & & \multirow{2}{*}{SA} & \multirow{2}{*}{PC} & \multirow{2}{*}{AVG} 
        & \multirow{2}{*}{Instr.}
        & \multicolumn{2}{c}{Common Sense} 
        & \multicolumn{2}{c}{Physics Adherence} 
        & \multirow{2}{*}{Total} \\
        \cmidrule(lr){7-8} \cmidrule(lr){9-10}
        & & & & & & Frame & Temporal & Mass & Penetr. & \\
        \midrule
        CogVideoX-2B~\cite{yang2024cogvideox}           & 2B   & 48.73 & 34.75 & 41.74 & 1.97 & 0.88 & 0.75 & 0.67 & 0.82 & 5.09 \\
        LTX~\cite{hacohen2024ltx}                    & 2B   & 21.10 & 25.09 & 23.10 & 1.57 & 0.74 & \underline{0.81} & \textbf{0.79} & \textbf{0.91} & 4.82 \\
        Allegro~\cite{zhou2024allegro}                & 2.8B & \underline{51.27} & 27.24 & 39.25 & 1.91 & 0.87 & 0.79 & \underline{0.76} & 0.89 & 5.22 \\
        CogVideoX-5B~\cite{yang2024cogvideox}           & 5B   & 47.30 & \textbf{42.79} & \textbf{45.05} & 2.01 & 0.91 & 0.80 & 0.75 & 0.86 & 5.33 \\
        Mochi~\cite{genmo2024mochi}                  & 10B  & \textbf{53.10} & 36.53 & \underline{44.81} & 1.95 & 0.74 & 0.70 & 0.65 & 0.87 & 4.91 \\
        \midrule
        Wan 2.1-T2V (base)~\cite{wan2025wan}     & 1.3B & 50.45 & 32.69 & 41.57 & \textbf{2.18} & \underline{0.93} & 0.80 & 0.72 & 0.83 & \underline{5.46} \\
        \textbf{\ourmethod (ours)} & 1.3B & 51.26 & \underline{38.16} & 44.71 & \underline{2.17} & \textbf{0.94} & \textbf{0.88} & \textbf{0.79} & \underline{0.90} & \textbf{5.68} \\
        \bottomrule
    \end{tabular}%
    }
\end{table*}

\subsection{Benchmarks and Metrics}
\label{sec:benchmarks_metrics}
We evaluate \ourmethod along two complementary axes---physical plausibility and overall visual quality---using established benchmarks. VideoPhy~\cite{bansal2024videophy} assesses Semantic Adherence (SA) and Physical Commonsense (PC) over 344 prompts. The standard protocol binarizes both scores; instead, we report soft values (averaged softmax outputs from the auto-rater) to better capture fine-grained differences between methods. WorldModelBench~\cite{li2025worldmodelbench} evaluates instruction following, common sense, and fine-grained physics adherence (Newton's laws, mass conservation, fluid dynamics, penetration, gravity) across 350 prompts of diverse physical scenarios. For WorldModelBench, we report the performances with respect to mass conservation and penetration as they pose the most critical challenges to models. The total score for WorldModelBench is a summation of all metric values for this benchmark.

\subsection{Hard vs. Random Negatives}\label{subsec:hardvsrand}
Table~\ref{tab:negative_ablation} examines each component under controlled conditions. The use of random negatives negatively impacts performance, as evidenced by a 5.43\% drop in VideoPhy SA and a decrease in WorldModelBench Total from 5.46 to 5.40. This outcome empirically confirms the gradient conflict discussed in Appendix~\ref{append:gradient_conflict}: when negatives significantly overlap with the positive in terms of velocity fields, the contrastive gradient can hinder reconstruction.

\begin{figure*}[t]
    \centering
    \includegraphics[width=\textwidth]{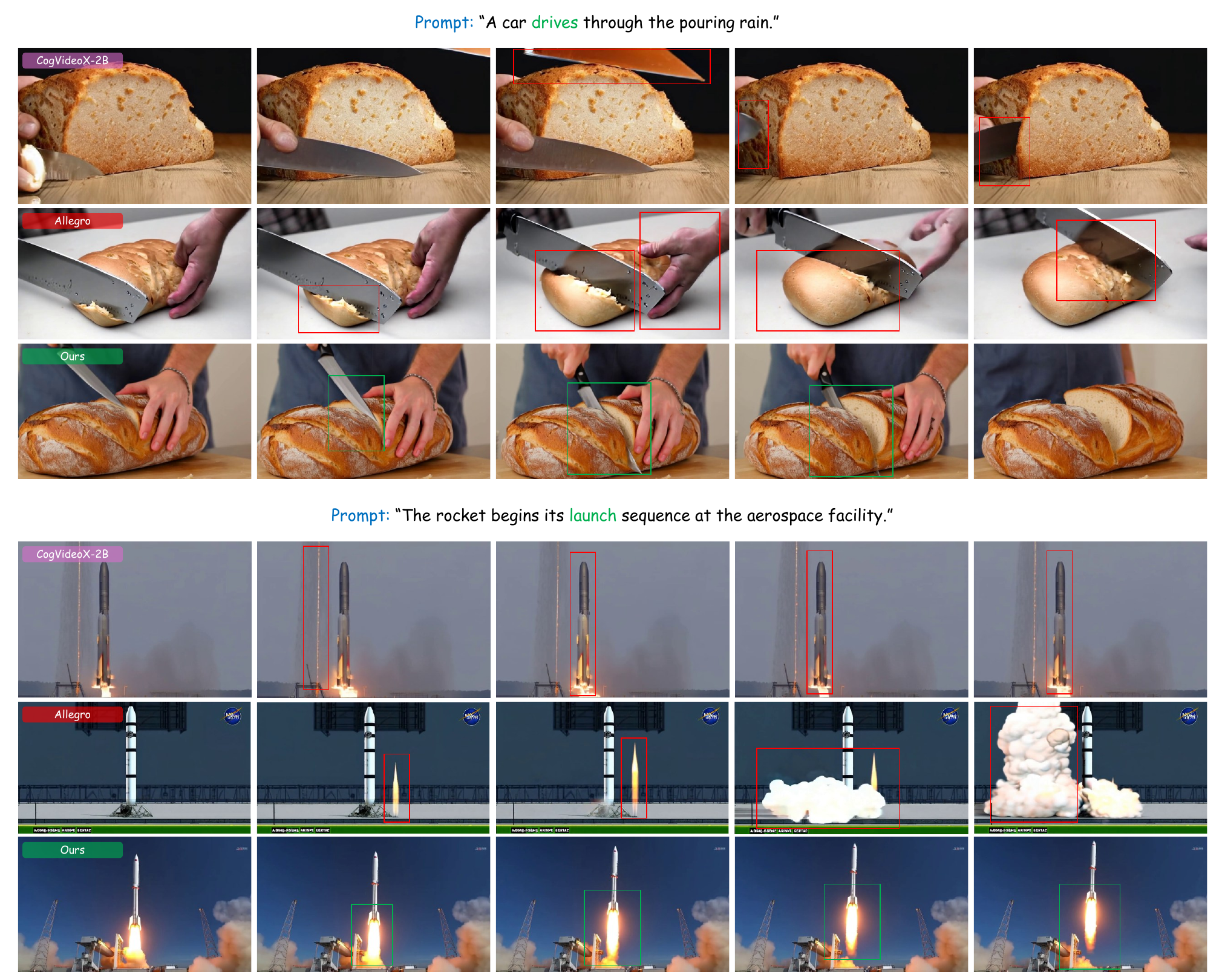}
    \caption{Comparison of \ourmethod with CogVideoX-2B and Allegro on a prompt from VideoPhy (top) and one from WorldModelBench (bottom).}
    \label{fig:sota_comp}
\end{figure*}

MaNS and MiNS tackle this issue at complementary scales. MaNS primarily enhances Physics Adherence, increasing Mass from 0.72 to 0.77 and Penetration from 0.83 to 0.88 by enabling interference-free partition separation. On the other hand, MiNS produces the largest improvements in Temporal Common Sense, raising the score from 0.80 to 0.83 through targeted physics perturbations. When combined, their effects surpass those of each component alone across all metrics, achieving a total score of 5.60 compared to 5.54 and 5.50 for the individual components, thereby confirming that the two signals are complementary.

Incorporating SFT results in the complete \ourmethod, which achieves the best scores on all metrics, with a VideoPhy average of 44.71\% and a Total score of 5.68, while also maintaining instruction-following capability (2.17 vs. 2.18 in zero-shot scenarios). This validates that each component contributes a unique and additive signal.

\subsection{Comparison with State-of-the-Art}

We compare our method with state-of-the-art models across a broad range, including those of the same and larger scales. Results are provided in Table~\ref{tab:sota_comparison}. Despite using only 1.3B parameters, \ourmethod achieves the highest WorldModelBench total score (5.68), surpassing models with significantly larger capacity such as Mochi (10B, 4.91) and CogVideoX-5B (5B, 5.33). This demonstrates that targeted physics-aware fine-tuning can be more effective than simply scaling model size for improving physical realism.
On VideoPhy, our method achieves an average score of 44.71\%, which is competitive with Mochi's score of 44.81\%, despite Mochi requiring 7.7 times more parameters. Additionally, we obtained the second-highest PC score of 38.16\% among all models, only behind CogVideoX-5B, which scored 42.79\% with 3.8 times our parameter count. These results indicate that contrastive flow matching, combined with physics-aware negative sampling, offers an efficient and effective way to incorporate physical understanding into video generation models without the need for an excessive number of parameters. A qualitative comparison of our model against state-of-the-art is presented in Figure~\ref{fig:sota_comp}.

\subsection{Training dynamics}
A key question for any post-training method is whether its learning signal remains valuable as training progresses or if it saturates once the model adapts to the fine-tuning distribution. We compare the convergence behavior of our method and supervised fine-tuning (SFT) on VideoPhy over 15,000 steps in Figure~\ref{fig:training_dynamics}. SFT shows steady improvement but plateaus after 5,000 steps, achieving only an increase of +0.32 in SA and +0.48 in PC in its best results. In contrast, our method continues to improve throughout the entire training process, with an increasing lead over SFT. This indicates the significance of our method in post-training.

\begin{figure*}[t]
    \centering
    \includegraphics[width=\textwidth]{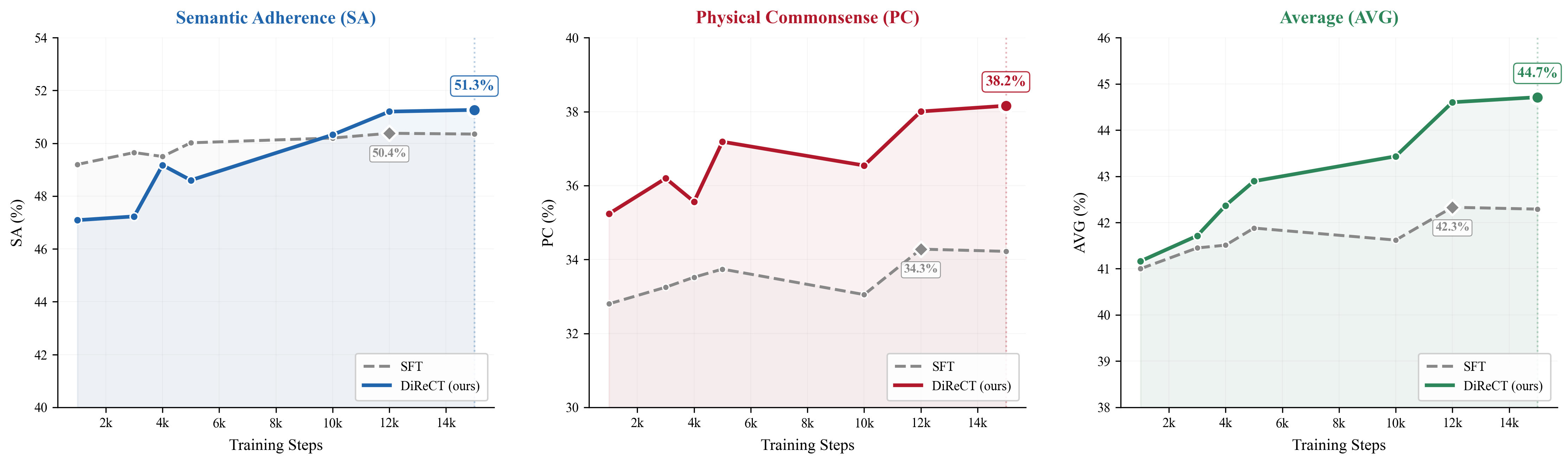}
    \caption{Training dynamics in terms of VideoPhy scores. SFT starts to saturate beyond 5k steps, while \ourmethod continues to improve and stabilizes around 15k steps, with the gap widening particularly on Physical Commonsense (PC).}
    \label{fig:training_dynamics}
\end{figure*}

\section{Conclusion}

We presented \ourmethod, a post-training framework that improves physical plausibility in flow-matching video generators through entanglement-aware contrastive learning. Our core insight is that naive contrastive flow matching fails under text conditioning because semantic--physics entanglement causes the contrastive gradient to oppose reconstruction. We formalized this conflict and proposed a two-scale solution: macro-contrastive sampling draws negatives from semantically distant partitions for interference-free trajectory separation, while micro-contrastive sampling constructs physics-perturbed hard negatives that isolate fine-grained physical distinctions. Applied to Wan 2.1-1.3B, \ourmethod achieves the highest WorldModelBench Total score among all compared models---including those with up to 7.7$\times$ more parameters---while preserving instruction-following capability and visual quality.



%
%
\bibliographystyle{splncs04}
\bibliography{main}

\clearpage
\appendix
\renewcommand{\theHsection}{A\arabic{section}} 
\setcounter{figure}{0}
\setcounter{table}{0}
\renewcommand{\thetable}{S\arabic{table}}  
\renewcommand{\thefigure}{S\arabic{figure}}



\section{Gradient Conflict Under Semantic Proximity} \label{append:gradient_conflict}

Consider the $\Delta$FM objective (Equation.~\ref{eq:delta_fm}).
At each training step, the model $v_\theta$ receives a flow-matching update that pulls its prediction toward the positive target velocity $\mathbf{u}^{+} = \dot{\alpha}_t \hat{x} + \dot{\sigma}_t \epsilon$, and a contrastive update that pushes its prediction away from the negative target velocity $\mathbf{u}^{-} = \dot{\alpha}_t \tilde{x} + \dot{\sigma}_t \tilde{\epsilon}$.
Writing $\mathbf{v} = v_\theta(x_t, t, y)$ for brevity, the effective gradient directions for these two signals are:
\begin{equation}\label{eq:fm_grad}
  \mathbf{g}_{\mathrm{FM}} \propto \mathbf{u}^{+} - \mathbf{v}
  \qquad \text{(flow matching: pull toward } \mathbf{u}^{+}\text{)},
\end{equation}
\begin{equation}\label{eq:contrast_grad}
  \mathbf{g}_{\mathrm{c}} \propto \mathbf{v} - \mathbf{u}^{-}
  \qquad \text{(contrastive: push away from } \mathbf{u}^{-}\text{)}.
\end{equation}
These two updates are compatible when they point in similar directions, and conflict when they oppose each other. Their alignment is captured by the following:

\begin{proposition}[Gradient conflict under semantic proximity]\label{prop:conflict}
Let\/ $\boldsymbol{\delta} = \mathbf{u}^{+} - \mathbf{u}^{-}$ denote the velocity gap between positive and negative targets. The inner product of the flow-matching and contrastive gradient directions satisfies
\begin{equation}\label{eq:alignment}
  \langle \mathbf{g}_{\mathrm{FM}},\, \mathbf{g}_{\mathrm{c}} \rangle
  \; \propto
  -\|\mathbf{u}^{+} - \mathbf{v}\|^2
  \;+\;
  \langle \mathbf{u}^{+} - \mathbf{v},\, \boldsymbol{\delta} \rangle.
\end{equation}
\end{proposition}

\begin{proof}
Substituting $\mathbf{u}^{-} = \mathbf{u}^{+} - \boldsymbol{\delta}$ into $\mathbf{g}_{\mathrm{c}}$:
\[
  \langle \mathbf{g}_{\mathrm{FM}}, \mathbf{g}_{\mathrm{c}} \rangle
  = \langle \mathbf{u}^{+} - \mathbf{v},\; \mathbf{v} - \mathbf{u}^{+} + \boldsymbol{\delta} \rangle
  = -\|\mathbf{u}^{+} - \mathbf{v}\|^2 + \langle \mathbf{u}^{+} - \mathbf{v},\, \boldsymbol{\delta} \rangle.
\]
\end{proof}

\noindent Equation~\ref{eq:alignment} reveals two competing terms:

\begin{enumerate}
  \item A \textbf{self-interference} term $-\|\mathbf{u}^{+} - \mathbf{v}\|^2$, which is always negative and represents the contrastive update working \emph{against} the flow-matching objective. This term dominates whenever the model has not yet converged ($\mathbf{v} \neq \mathbf{u}^{+}$).
  \item A \textbf{separation} term $\langle \mathbf{u}^{+} - \mathbf{v},\, \boldsymbol{\delta} \rangle$, which is positive when the velocity gap $\boldsymbol{\delta}$ is aligned with the flow-matching residual. This term provides a useful contrastive signal.
\end{enumerate}

\noindent The gradients are aligned ($\langle \mathbf{g}_{\mathrm{FM}}, \mathbf{g}_{\mathrm{c}} \rangle > 0$) only when the separation term dominates:
\begin{equation}\label{eq:alignment_condition}
  \langle \mathbf{u}^{+} - \mathbf{v},\, \boldsymbol{\delta} \rangle > \|\mathbf{u}^{+} - \mathbf{v}\|^2.
\end{equation}
When the anchor and negative conditions are semantically similar, their velocity fields share structure and $\|\boldsymbol{\delta}\| \to 0$, making the left-hand side vanish while the right-hand side remains positive. The inner product becomes $\langle \mathbf{g}_{\mathrm{FM}}, \mathbf{g}_{\mathrm{c}} \rangle \to -\|\mathbf{u}^{+} - \mathbf{v}\|^2 < 0$: the contrastive gradient directly opposes flow matching. Conversely, semantically distant negatives produce large $\|\boldsymbol{\delta}\|$, satisfying Equation~\ref{eq:alignment_condition} and yielding a cooperative training signal.

\paragraph{Implication for negative sampling.}
The analysis suggests a two-regime design. \emph{Distant} negatives (large $\|\boldsymbol{\delta}\|$) provide clean, globally cooperative contrastive gradients---suitable for establishing broad trajectory structure. \emph{Proximal} negatives are useful only when they differ along a specific, controlled axis (so that $\boldsymbol{\delta}$, though small in norm, is concentrated on the physics-relevant subspace), ensuring the separation term remains meaningful despite semantic overlap. This motivates the macro--micro decomposition described next.

\begin{figure*}[!htb]
    \centering
    \includegraphics[width=\textwidth]{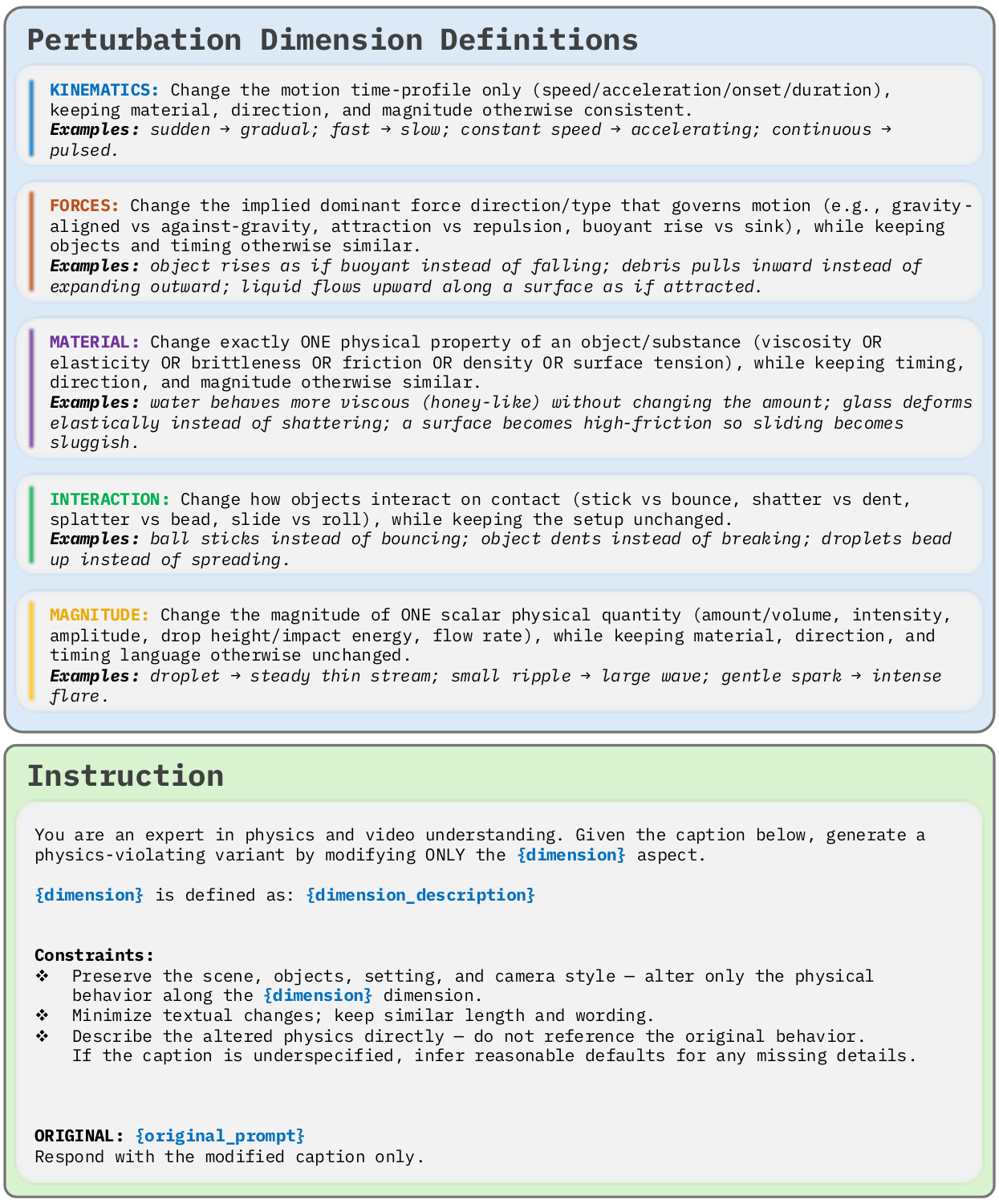}
    \caption{The prompt template used to generate the hard negatives.}
    \label{fig:prompt_template}
\end{figure*}

\begin{remark}[Why physics-concentrated $\boldsymbol{\delta}$ avoids conflict]
\label{rem:micro_alignment}
Proposition~\ref{prop:conflict} does not preclude useful contrastive 
learning at small $\|\boldsymbol{\delta}\|$; it requires only that the separation 
term $\langle \mathbf{u}^+ - \mathbf{v},\, \boldsymbol{\delta} \rangle$ dominate 
the self-interference term $\|\mathbf{u}^+ - \mathbf{v}\|^2$ 
(Equation~\ref{eq:alignment_condition}). The key observation is that for a 
well-pretrained video generator, the flow-matching residual 
$\mathbf{u}^+ - \mathbf{v}$ is itself concentrated on the \emph{physics-relevant 
subspace}: the base model already captures scene semantics, layout, and 
appearance with high fidelity, so its remaining prediction error is dominated 
by physically implausible dynamics (e.g., incorrect accelerations, 
inter-penetrations). When MiNS constructs a negative whose velocity gap 
$\boldsymbol{\delta}$ is aligned with this same subspace, the inner product 
$\langle \mathbf{u}^+ - \mathbf{v},\, \boldsymbol{\delta} \rangle \approx 
\|\mathbf{u}^+ - \mathbf{v}\|\ \|\boldsymbol{\delta}\|$ approaches its 
Cauchy--Schwarz upper bound, making Equation~\ref{eq:alignment_condition} 
satisfiable even at small $\|\boldsymbol{\delta}\|$---provided the residual 
norm $\|\mathbf{u}^+ - \mathbf{v}\|$ is itself small, as expected after 
large-scale pretraining. More precisely, denoting by $P_{\mathrm{phys}}$ the 
projection onto the physics-relevant subspace, the alignment condition 
relaxes to
\begin{equation}
\label{eq:relaxed_alignment}
\|P_{\mathrm{phys}}(\mathbf{u}^+ - \mathbf{v})\|\;\|\boldsymbol{\delta}\| 
\;\gtrsim\; \|P_{\mathrm{phys}}(\mathbf{u}^+ - \mathbf{v})\|^2 
+ \|P_{\perp}(\mathbf{u}^+ - \mathbf{v})\|^2,
\end{equation}
where $P_{\perp} = I - P_{\mathrm{phys}}$. Since pretraining minimizes the 
semantic residual $\|P_{\perp}(\mathbf{u}^+ - \mathbf{v})\|$, the 
right-hand side is dominated by the physics residual, and the condition 
reduces approximately to $\|\boldsymbol{\delta}\| \gtrsim 
\|P_{\mathrm{phys}}(\mathbf{u}^+ - \mathbf{v})\|$---a much weaker 
requirement than the $\|\boldsymbol{\delta}\| \gtrsim \|\mathbf{u}^+ - 
\mathbf{v}\|$ needed for unstructured negatives. This is the regime 
targeted by MiNS: small but physics-aligned $\boldsymbol{\delta}$, 
operating on a model whose residual error is concentrated in the same 
subspace.
\end{remark}

\section{Perturbation Instructions, Examples, and Minimality Statistics}
\label{append:perturbation_examples}

To generate physics-perturbed prompts for hard negative video synthesis, we use the template shown in Figure~\ref{fig:prompt_template}. The targeted physics dimension and its definition are provided to Qwen2.5-7B-Instruct alongside a set of constraints (Rules 1–4 in the template). Without these constraints, the LLM tends to reference both the original and modified physics (e.g., "the person runs instead of walking"), producing prompts that leak the anchor's physical behavior. The constraints also reduce large-scale semantic rewrites that would alter scene content beyond the targeted axis.
For each anchor prompt, we generate 10 candidate perturbations with the physics dimension sampled uniformly at random. We then compute text embeddings via the frozen T5 encoder of Wan 2.1, apply mean pooling, and measure cosine similarity to the anchor embedding. Candidates with similarity below 0.87 are discarded, as these typically reflect modifications to scene content rather than isolated physics changes. From the surviving candidates, we select the top 3 by cosine similarity, prioritizing the most semantically conservative perturbations to maximize physics isolation. This yields 33,897 hard-negative prompts in the training set.

\section{Sensitivity to Loss Weighting Coefficients}
\label{append:lambda_sensitivity}

We tune the loss coefficients in a sequential, greedy fashion: first $\lambda_{\mathrm{anc}}$ (with contrastive terms disabled), then $\lambda_{\mathrm{rand}}$ and $\lambda_{\mathrm{hard}}$ in turn, each time fixing the previously selected values.
 
\paragraph{Distributional anchoring strength ($\lambda_{\mathrm{anc}}$).}
We first identify an anchoring coefficient that enables successful SFT without catastrophic forgetting. As shown in Table~\ref{tab:lambda_anc}, $\lambda_{\mathrm{anc}}{=}0.2$ yields the best VideoPhy average; lower values under-regularize (SA drops sharply), while higher values over-constrain the model, limiting PC gains.

\begin{table}[t]
\centering
\caption{\textbf{Sensitivity to distributional anchoring strength $\lambda_{\mathrm{anc}}$.}
SFT only, no contrastive terms.}
\label{tab:lambda_anc}
\setlength{\tabcolsep}{12pt}
\renewcommand{\arraystretch}{1.25}
\begin{tabular}{c ccc}
\toprule
\multirow{2}{*}{$\lambda_{\mathrm{anc}}$} & \multicolumn{3}{c}{\textbf{VideoPhy} (\%) $\uparrow$} \\
\cmidrule(lr){2-4}
 & SA & PC & AVG \\
\midrule
0.1 & 48.57 & 33.50 & 41.04 \\
\rowcolor{gray!10}
\textbf{0.2} & \textbf{50.38} & \textbf{34.28} & \textbf{42.33} \\
0.3 & 50.40 & 31.69 & 41.05 \\
\bottomrule
\end{tabular}
\end{table}

\paragraph{Macro-contrastive strength ($\lambda_{\mathrm{rand}}$).}
With $\lambda_{\mathrm{anc}}{=}0.2$ and $\lambda_{\mathrm{hard}}{=}0.01$ fixed, we sweep the MaNS coefficient. Table~\ref{tab:lambda_rand} shows a clear trade-off: too small a value ($0.001$) provides negligible trajectory separation, while too large a value ($0.02$) over-separates trajectories, causing SA to degrade. The selected value $\lambda_{\mathrm{rand}}{=}0.005$ balances both metrics.

\begin{table}[t]
\centering
\caption{\textbf{Sensitivity to macro-contrastive strength $\lambda_{\mathrm{rand}}$.}
$\lambda_{\mathrm{anc}}{=}0.2$, $\lambda_{\mathrm{hard}}{=}0.01$.}
\label{tab:lambda_rand}
\setlength{\tabcolsep}{12pt}
\renewcommand{\arraystretch}{1.25}
\begin{tabular}{c ccc}
\toprule
\multirow{2}{*}{$\lambda_{\mathrm{rand}}$} & \multicolumn{3}{c}{\textbf{VideoPhy} (\%) $\uparrow$} \\
\cmidrule(lr){2-4}
 & SA & PC & AVG \\
\midrule
0.001 & 50.52 & 35.48 & 43.00 \\
\rowcolor{gray!10}
\textbf{0.005} & \textbf{50.65} & 36.47 & \textbf{43.56} \\
0.02  & 49.83 & \textbf{36.82} & 43.33 \\
\bottomrule
\end{tabular}
\end{table}
 
\paragraph{Micro-contrastive strength ($\lambda_{\mathrm{hard}}$).}
With $\lambda_{\mathrm{anc}}{=}0.2$ and $\lambda_{\mathrm{rand}}{=}0.005$ fixed, we sweep the MiNS coefficient. As shown in Table~\ref{tab:lambda_hard}, $\lambda_{\mathrm{hard}}$ tolerates a wider range than $\lambda_{\mathrm{rand}}$ because the physics-perturbed negatives concentrate $\boldsymbol{\delta}$ in the physics-relevant subspace, reducing gradient conflict by construction (Remark~\ref{rem:micro_alignment}). Nevertheless, excessively large values ($0.1$) destabilize training, degrading both SA and PC.

\begin{table}[t]
\centering
\caption{\textbf{Sensitivity to micro-contrastive strength $\lambda_{\mathrm{hard}}$.}
$\lambda_{\mathrm{anc}}{=}0.2$, $\lambda_{\mathrm{rand}}{=}0.005$.}
\label{tab:lambda_hard}
\setlength{\tabcolsep}{12pt}
\renewcommand{\arraystretch}{1.25}
\begin{tabular}{c ccc}
\toprule
\multirow{2}{*}{$\lambda_{\mathrm{hard}}$} & \multicolumn{3}{c}{\textbf{VideoPhy} (\%) $\uparrow$} \\
\cmidrule(lr){2-4}
 & SA & PC & AVG \\
\midrule
0.005 & 50.61 & 35.97 & 43.29 \\
\rowcolor{gray!10}
\textbf{0.02} & \textbf{50.78} & \textbf{37.46} & \textbf{44.12} \\
0.1  & 49.18 & 37.03 & 43.11 \\
\bottomrule
\end{tabular}
\end{table}

\section{Training and Inference Setup}
\label{append:train_inference}

We fine-tune the full transformer network in \texttt{bf16} mixed precision. Loss values greater than 50 are masked out to stabilize training, as recommended by the VideoX-Fun framework. No learning rate warm-up or decay schedule is used; the constant rate of $1 \times 10^{-6}$ provides sufficient stability given the regularizing effect of the distributional anchoring term. The frozen reference model $\theta_{\mathrm{ref}}$ shares the forward pass with the positive samples only, so the anchoring loss (Equation~\ref{eq:kl_anchor}) introduces no additional inference cost per step for negatives.

For semantic clustering in MaNS, we run $k$-means with 50 random restarts on mean-pooled T5 embeddings and select the partition with the lowest inertia. Cluster assignments are computed once before training and remain fixed throughout. At inference time, we use the Euler Discrete Scheduler with 50 denoising steps and a classifier-free guidance scale of 5.0. No test-time augmentation, prompt refinement, or ensembling is applied. Each video generation takes approximately 85 seconds on a single H200 GPU. For all benchmark evaluations, we generate one video per prompt using a fixed random seed.

\section{Training Cost and Memory Usage}
\label{append:training_cost}

A practical advantage of contrastive flow matching over preference-based alternatives is its minimal training overhead.
DPO-based methods require each negative sample to pass through both the trainable model~$\theta$ and a frozen reference
model~$\theta_{\text{ref}}$, effectively doubling forward-pass cost and memory. In \ourmethod, negatives contribute only
precomputed target velocities~$\mathbf{u}^{-}$ to the contrastive loss (Eqs.~\ref{eq:loss_rand},~\ref{eq:loss_hard});
no additional forward pass is needed. The single reference-model pass used for distributional anchoring (Eq.~\ref{eq:kl_anchor}) is shared with the positive sample, equivalent to the KL penalty in standard regularized fine-tuning.

Hard-negative videos and their latent representations are generated offline as a one-time preprocessing step, requiring
approximately \textbf{13}~GPU-hours on 4~H200 GPUs. Once precomputed, per-step cost is nearly identical to SFT in wall-clock time, as shown in Table~\ref{tab:training_cost}. Memory increases by ${\sim}7$\,GB ($+31\%$) due to loading the cached negative
latents, but remains well within single-GPU capacity and substantially below the footprint a full DPO setup would
require (which must maintain a complete copy of $\theta_{\text{ref}}$ in memory throughout training).

\begin{table}[t]
\centering
\caption{\textbf{Per-iteration training cost.}
Both methods use Wan 2.1-1.3B at $480 \times 832$ resolution (81 frames), global batch size 8, DeepSpeed ZeRO-2 on H200 GPUs.}
\label{tab:training_cost}
\setlength{\tabcolsep}{12pt}
\renewcommand{\arraystretch}{1.25}
\begin{tabular}{l cc}
\toprule
\multirow{2}{*}{\textbf{Method}} & \multicolumn{2}{c}{\textbf{Training Cost} $\downarrow$} \\
\cmidrule(lr){2-3}
 & Mem.\,(GB) & Time\,(s/iter) \\
\midrule
SFT         & 23.29 & 4.94 \\
\ourmethod  & 30.45 & 4.95 \\
\bottomrule
\end{tabular}
\end{table}

\section{Convergence Analysis and Training Signal Longevity}
\label{append:training_dynamics}

A central concern for any post-training method is whether its learning signal provides sustained value throughout optimization or saturates early once the model adapts to the fine-tuning distribution. If the signal saturates, additional training steps yield diminishing returns and risk overfitting to the fine-tuning data without further improving the target capability. We investigate this question by comparing the convergence behavior of DiReCT and supervised fine-tuning (SFT) on VideoPhy over 15,000 training steps, evaluating at regular intervals. Both methods use identical training data, base model, and optimization hyperparameters (Section~\ref{sec:impl}); the only difference is the loss function.

\paragraph{SFT saturates early.} As shown in Figure~\ref{fig:training_dynamics}, SFT exhibits rapid initial improvement during the first 3,000--5,000 steps as the model adapts to the fine-tuning data distribution. Beyond this point, however, progress stalls: the best SA and PC scores achieved by SFT represent gains of only +0.32 and +0.48, respectively, over the zero-shot baseline. This plateau is consistent with the nature of the reconstruction objective—once the model has learned to reproduce the training videos with low per-frame error, the loss provides no additional gradient signal for distinguishing physically plausible dynamics from implausible ones. The velocity field converges toward the conditional mean of nearby trajectories, which is precisely the mode-averaging behavior identified in Section~\ref{sec:multi_scale} as the root cause of physical violations.

\paragraph{DiReCT provides a sustained learning signal.} In contrast, DiReCT continues to improve steadily throughout the entire 15,000-step training window. This sustained improvement can be attributed to the complementary nature of its loss components. The flow-matching term drives early convergence on reconstruction quality, similar to SFT. As the model improves and the reconstruction residual $\|\mathbf{u}^+ - \mathbf{v}\|$ shrinks, the contrastive terms become increasingly effective: the alignment condition (Equation~\ref{eq:alignment_condition}) is more easily satisfied when the residual is small (Remark~\ref{rem:micro_alignment}), allowing the macro- and micro-contrastive gradients to steer the velocity field toward physics-consistent trajectories without opposing reconstruction. This creates a natural curriculum: reconstruction first, then physics refinement.

\paragraph{The gap widens over time.} A notable feature of Figure~\ref{fig:training_dynamics} is that the performance gap between DiReCT and SFT does not merely persist but actively \emph{widens} as training progresses. This divergence is most pronounced on Physical Commonsense (PC), where the gap grows from a marginal difference at 3,000 steps to a substantial lead by 15,000 steps. SA exhibits a similar but more moderate trend, consistent with our earlier observation that semantic quality is primarily governed by the reconstruction term (which both methods share) while physics discrimination depends on the contrastive signal unique to DiReCT. The widening gap provides evidence that the contrastive objective introduces a genuinely complementary learning signal—one that the reconstruction loss alone cannot replicate regardless of training duration.

\paragraph{Stabilization without overfitting.} DiReCT stabilizes around 15,000 steps, with scores plateauing rather than declining. This indicates that the distributional anchoring term (Equation~\ref{eq:kl_anchor}) successfully prevents catastrophic forgetting even over extended training, keeping the model within the pretrained generation manifold while allowing targeted physics improvements. The absence of performance degradation at convergence further validates our choice of $\lambda_{\mathrm{anc}} = 0.2$ (see Appendix~\ref{append:lambda_sensitivity} for sensitivity analysis).


\section{Sensitivity to Semantic Partition Granularity}
\label{append:clustering}

Table~\ref{tab:cluster_sensitivity} reports VideoPhy scores as we vary the number of semantic clusters $K$ used in macro-contrastive negative sampling (MaNS). Performance follows a clear inverted-U trend, peaking at $K{=}32$ across all three attributes. At coarse granularities ($K{=}8$), individual clusters span semantically diverse prompts, so partition-exclusive negatives can still share substantial velocity-field structure with the anchor. This violates the large-$\|\boldsymbol{\delta}\|$ assumption underlying Proposition~1, reintroducing the gradient conflict that MaNS is designed to prevent; accordingly, Physical Commonsense drops sharply to 34.52\%, a reduction of 3.64 percentage points relative to the optimum. Doubling to $K{=}16$ narrows intra-cluster diversity and recovers much of the gap, yet residual semantic overlap within clusters still limits the contrastive signal.

At fine granularities ($K{=}64, 128$), clusters become small and their boundaries increasingly driven by embedding noise rather than genuine semantic structure. This fragmentation reduces the effective diversity of the excluded negative pool and introduces noisy partition assignments that weaken the theoretical guarantee of cross-cluster separation. The decline, however, is more gradual than in the low-$K$ regime: $K{=}64$ retains an AVG of 43.92\%, only 0.79 points below the optimum, suggesting that even imperfect partitions still provide a useful coarse separation signal. $K{=}128$ degrades further to 42.95\%, confirming that excessive fragmentation is detrimental.

Notably, Semantic Adherence (SA) remains relatively stable across all settings (49.87--51.26\%), consistent with SA being driven primarily by the reconstruction term rather than the contrastive objective. Physical Commonsense (PC), which depends most directly on clean trajectory separation, exhibits the strongest sensitivity to $K$, spanning a 3.64-point range. This asymmetry corroborates our claim that MaNS primarily targets physics-relevant structure in the velocity field. Based on these results, we adopt $K{=}32$ as the default throughout all experiments.

\begin{table}[t]
\centering
\caption{\textbf{Effect of cluster count $K$ in macro-contrastive sampling (MaNS).}
We vary the number of semantic clusters used for group-exclusive negative
sampling. Too few clusters risk intra-cluster semantic diversity (gradient
conflict); too many produce overly fragmented groups with noisy boundaries.}
\label{tab:cluster_sensitivity}
\setlength{\tabcolsep}{12pt}
\renewcommand{\arraystretch}{1.25}
\begin{tabular}{c ccc}
\toprule
\multirow{2}{*}{$K$} & \multicolumn{3}{c}{\textbf{VideoPhy} (\%) $\uparrow$} \\
\cmidrule(lr){2-4}
 & SA & PC & AVG \\
\midrule
8    & 49.87 & 34.52 & 42.20 \\
16   & 50.61 & 36.38 & 43.50 \\
\rowcolor{gray!10}
\textbf{32}   & \textbf{51.26} & \textbf{38.16} & \textbf{44.71} \\
64   & 50.94 & 36.89 & 43.92 \\
128  & 50.18 & 35.71 & 42.95 \\
\bottomrule
\end{tabular}
\end{table}

\begin{table}[t]
\centering
\caption{%
  Report of gradient alignment between flow-matching and contrastive loss terms for $\Delta$FM and \ourmethod according to the cosine similarity metric.
}
\label{tab:gradient_alignment}
\small
\setlength{\tabcolsep}{4pt}
\begin{tabular}{@{}l c@{}}
\toprule
\textbf{Method}
  & \textbf{Cosine} $\uparrow$ \\
\midrule
$\Delta$FM (random neg.)
  &$0.12 \pm 0.21$ \\
\ourmethod{}-MaNS (ours)
  &$-0.05 \pm 0.12$ \\
\bottomrule
\end{tabular}
\end{table}

\section{Gradient Alignment Measurement}\label{append:gradient_alignment}

To evaluate the effectiveness of our proposed sampling approach in resolving the semantic/physics entanglement, we define a set of metrics that mainly work with gradients of the flow-matching term on the positive sample and the contrastive term for random negatives (MaNS). At each training step, we perform two independent backward passes through the set of trainable parameters~$\phi$:

\begin{enumerate}
    \item A backward pass on the flow-matching loss $\mathcal{L}_{\mathrm{FM}}$, producing $\mathbf{g}_{\mathrm{FM}} = \nabla_\phi \mathcal{L}_{\mathrm{FM}}$. 
    \item A backward pass on the weighted contrastive loss $\lambda\mathcal{L}_{\mathrm{c}}$, producing
      $\mathbf{g}_{\mathrm{c}} = \nabla_\phi(\lambda \mathcal{L}_{\mathrm{c}})
       = \lambda\,\nabla_\phi \mathcal{L}_{\mathrm{c}}$. Here, the contrastive loss refers to the loss penalizing velocity field similarity for random negatives. 
\end{enumerate}

\noindent
Because the total loss is $\mathcal{L}_{\mathrm{total}} = \mathcal{L}_{\mathrm{FM}} - \lambda\mathcal{L}_{\mathrm{c}}$, the combined parameter update direction is

\begin{equation}
\label{eq:g_total}
\mathbf{g}_{\mathrm{total}}
  = \mathbf{g}_{\mathrm{FM}} - \mathbf{g}_{\mathrm{c}}.
\end{equation}

To ensure fairness, we utilize the same dataset and hyperparameters, calculating the metrics over the first 1000 steps of the training process. In the next section, we first describe the metrics and then provide a comprehensive comparison of \(\Delta\)FM and \ourmethod across these metrics. 

A natural measure of gradient interaction is the cosine similarity between the flow-matching and contrastive gradients:
\begin{equation}
\label{eq:cosine_fm_c}
\mathrm{cos}(\mathbf{g}_{\mathrm{FM}},\,\mathbf{g}_{\mathrm{c}})
  = \frac{\mathbf{g}_{\mathrm{FM}} \cdot \mathbf{g}_{\mathrm{c}}}
         {\|\mathbf{g}_{\mathrm{FM}}\|\;\|\mathbf{g}_{\mathrm{c}}\|}.
\end{equation}
This metric captures directional alignment of gradients. Because the contrastive term is subtracted in Equation~\eqref{eq:g_total}, a positive cosine indicates that the two raw gradients point in the same direction, so the subtraction partially cancels $\mathbf{g}_{\mathrm{FM}}$---constituting opposition in the combined update---while a negative cosine indicates that the subtraction reinforces $\mathbf{g}_{\mathrm{FM}}$. 

As established in Proposition~\ref{prop:conflict}, the self-interference term $-\|\mathbf{u}^{+} - \mathbf{v}\|^2$ in
Equation~\ref{eq:alignment} introduces an inherent level of opposition between the contrastive and flow-matching gradients
that persists regardless of the sampling strategy.  The critical question is therefore not whether opposition exists, but whether
The contrastive term provides sufficient complementary information to justify this cost.  A cosine of $+1$ represents the worst case: the contrastive gradient is fully aligned with the flow-matching gradient, so the subtraction directly opposes reconstruction without introducing any new learning signal, pure cost with no benefit.  A cosine of $-1$ represents the opposite degenerate case: the contrastive term merely reinforces the flow-matching direction, providing redundant information rather than a complementary signal.  The most productive regime lies between these extremes: the contrastive gradient deviates from the flow-matching direction sufficiently to introduce complementary information, such as physics-specific corrections, while the inherent opposition remains bounded relative to the useful orthogonal signal. 
As reported in Table~\ref{tab:gradient_alignment}, $\Delta$FM exhibits a mean cosine of $+0.12$, confirming that random negatives induce the gradient opposition predicted by Proposition~\ref{prop:conflict}: the contrastive and flow-matching gradients partially align, so the subtraction in Eq.~\eqref{eq:g_total} actively opposes reconstruction. \ourmethod{} shifts this to $-0.05$, crossing into the cooperative regime where the subtracted contrastive term reinforces rather than counteracts the flow-matching update. This sign reversal indicates that the structured negatives produced by MaNS redirect the contrastive gradient away from the reconstruction axis and toward complementary, physics-specific corrections. The nearly twofold reduction in standard deviation ($0.12$ vs.\ $0.21$) further suggests that \ourmethod produces more stable gradient interactions across training steps.


\section{Additional Qualitative Examples}
\label{append:qualitative}

Figure~\ref{fig:oursvsbaseline2} provides additional comparisons between DiReCT and the baseline across diverse physical scenarios beyond those shown in the main paper.

\begin{figure*}[t]
    \centering
    \includegraphics[width=\textwidth]{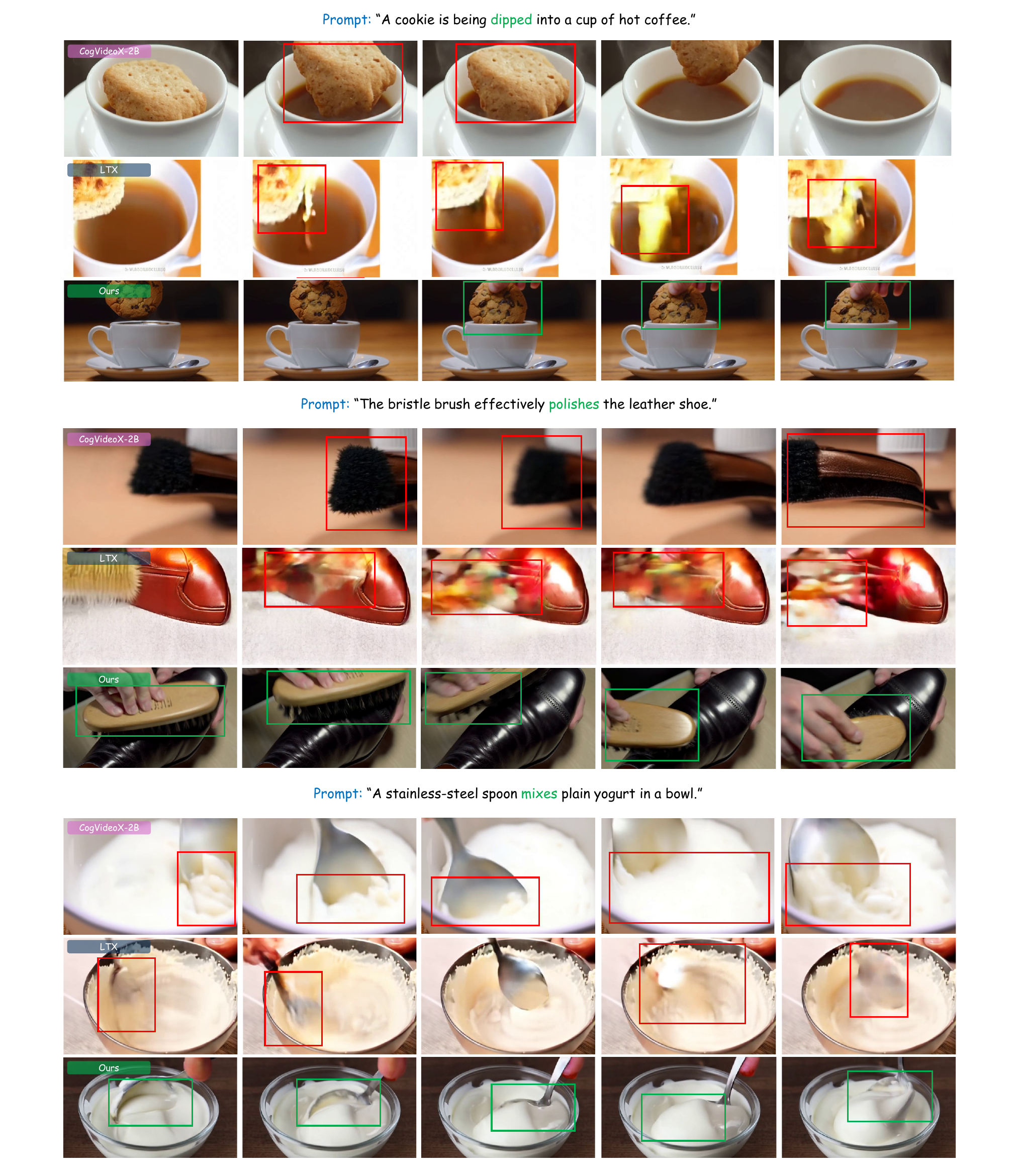}
    \caption{\textbf{Additional qualitative comparisons with CogVideoX-2B and
LTX on VideoPhy prompts.} Red boxes highlight physical
violations in baseline outputs; green boxes indicate
physically consistent behavior produced by \ourmethod{}.
In the top example, LTX causes the cookie to
dissolve and lose structural integrity upon contact with
the coffee, while CogVideoX-2B drops the cookie with no hand
visible in frame; \ourmethod{} preserves the cookie's
shape throughout the dipping motion with a plausible
hand--cookie interaction. In the middle example,
CogVideoX-2B fails to render the shoe entirely, showing
only a brush, and LTX exhibits severe visual degradation
with color bleeding and loss of object boundaries;
\ourmethod{} maintains coherent brush--shoe contact with a
realistic polishing motion. In the bottom example,
LTX produces unnatural spoon--yogurt
interactions with the spoon merging into the mixture,
while CogVideoX-2B generates implausible yogurt rigidity;
\ourmethod{} produces a physically plausible stirring
motion in which the spoon and yogurt maintain distinct
material properties throughout.}
    \label{fig:oursvsbaseline2}
\end{figure*}

\end{document}